\documentclass{article}
\usepackage{arxiv}
\usepackage[british]{babel}

\usepackage{booktabs,url,amsmath,algorithm2e,hyperref,microtype,stmaryrd,enumitem,bm,amssymb,mathtools,tikz,pgfplots,multirow,amsthm,caption,subcaption,natbib}
\pgfplotsset{compat=1.18}
\usetikzlibrary{arrows,plotmarks,decorations.markings,trees,shapes,pgfplots.groupplots,automata,circuits}
\tikzset{dot/.style = {circle, fill, minimum size=#1,inner sep=0pt, outer sep=0pt, fill, circle},dot/.default = 6pt}
\tikzset{dot2/.style = {circle, fill, color=black!40,minimum size=6pt,inner sep=0pt, outer sep=0pt, fill, circle}}
\tikzstyle{a}=[->,>=stealth,dashed]
\tikzstyle{a2}=[->,>=stealth]

\newtheorem{proposition}{Proposition}
\newtheorem{theorem}{Theorem}
\newtheorem{corollary}{Corollary}
\newtheorem{lemma}{Lemma}

\title{Partial Identification under Causal Orders by Linear Programming}
\date{}

\author{
	\href{https://orcid.org/0009-0000-4174-3597}{\includegraphics[scale=0.06]{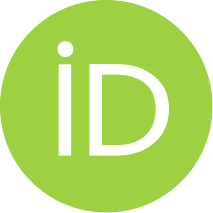}\hspace{1mm}Eric Rossetto} 
	\quad \textnormal{and} \quad 
	\href{https://orcid.org/0000-0001-7915-2768}{\includegraphics[scale=0.06]{orcid.pdf}\hspace{1mm}Alessandro Antonucci} \\[1.5ex]
	Istituto Dalle Molle di Studi sull'Intelligenza Artificiale (IDSIA)\\
	Scuola Universitaria Professionale della Svizzera Italiana (SUPSI)\\
	Lugano, Switzerland \\
	\texttt{\{eric.rossetto, alessandro.antonucci\}@supsi.ch} \\
}

\hypersetup{
pdftitle={Partial Identification under Causal Orders by Linear Programming},
pdfsubject={cs.AI},
pdfauthor={Eric Rossetto, Alessandro Antonucci},
pdfkeywords={Structural Causal Models, Partial Identifiability, Linear Programming},
}

\begin{document}
\maketitle
\begin{abstract}
Non-parametric (partial) identification of counterfactual queries typically relies on a fully specified causal graph. Motivated by settings with incomplete domain knowledge, we challenge this requirement by leveraging structural assumptions that are inherently implied by the query itself. We show that any counterfactual inquiry induces a, mostly partial, topological ordering over relevant variables, which, in turn, enables an explicit query parametrisation reducing the identification task to a linear program. This allows bounding arbitrary counterfactual and nested counterfactual queries. Our work can be viewed as a generalisation of the classical bounding framework of Tian and Pearl (2000), originally developed for probabilities of causation. We also prove the \emph{tightness} of our bounds by constructing structural causal models that attain the bounds whilst being compatible with both the observed data and the query-implied order. To assess both the generality and practical utility of the proposed bounding procedure, we revisit several case studies from the literature, demonstrating how the derived bounds can be used to yield informative insights even in the absence of an input causal graph.
\end{abstract}

\section{Introduction}\label{sec:introduction}
Causal reasoning deepens our understanding of the effects of policies, actions, and treatments in domains that require accountable decision-making, such as algorithmic fairness, personalised medicine, economics, and the broader social sciences. Quantifying such causal inquiries demands a rigorous mathematical foundation. The work of \citet{pearlCausalityModelsReasoning2009} formalised the systems under study using structural causal models (SCMs), which elegantly encode prior causal beliefs via directed graphs \citep{pearlCausalDiagrams1995}. Ideally, these structural assumptions allow causal queries to be uniquely computed---formally \emph{identified}---from the available set of observations. In many practical settings, however, the available assumptions are insufficient for point identification, necessitating \emph{partial identification}: computing a range within which the target quantity is guaranteed to lie.

Early work by \citet{manskiNonParametricBoundTreatmentEffects1990} derived closed-form bounds for simple treatment effect estimation without relying on structural assumptions. Similarly, \citet{tianProbabilitiesCausationBounds2000} formulated a linear program to compute sharp bounds on the \emph{probabilities of causation} \citep{pearlProbabilitiesCausationThree1999} from experimental and observational data under either general or mild assumptions. Whilst weak structural assumptions typically yield substantially wide bounds, assuming a fully specified causal graph might produce tighter intervals. Graph-based reductions to linear programs were first developed for the instrumental variable setting \citep{balkeCounterfactualProbabilitiesComputational1994, balkeProbabilisticEvaluationCounterfactual1994} and later extended to generalised instrumental constraints \citep{sachsGeneralMethodDeriving2023}. For general graphical assumptions, partial identification is reduced to a polynomial program via \emph{canonicalisation}, though the exponential growth of the resulting exogenous state space ultimately necessitates approximate methods \citep{zhangPartialCounterfactualIdentification2021, duarte2024automated}.

We revisit partial identification in the spirit of \citet{tianProbabilitiesCausationBounds2000}, whose bounding method did not require a specification of a causal graph; instead, their proposal relied on a direct parametrisation of the joint counterfactual distribution. To extend this framework to arbitrary queries, we leverage the minimal structural assumptions inherently implied by the query itself, reducing the identification task to a linear program without imposing additional structural commitments. Our specific contributions are as follows: (i)~we reduce partial identification under a total order of variables to a linear program, capable of handling arbitrary counterfactual, including nested, queries; (ii)~we prove that the resulting bounds are \emph{tight} by constructing witnessing structural causal models that attain them; (iii)~we prove that the order-dependent optimisation task is invariant under both the marginalisation of unqueried variables and the specific choice of total order (compatible with the query); and (iv)~we show that the linear program admits a lifted reformulation by exploiting symmetries induced by the observations and the query's logical evaluations. The presented results define an automated framework that enables the bounding of arbitrary queries without the need for explicit graphical assumptions. It is important to emphasise that graph-based methods generally yield narrower bounds by exploiting conditional independencies; in contrast, we pose ourselves in a regime of incomplete domain knowledge, abstaining from such graphical commitments.

After introducing the necessary background in Sect.~\ref{sec:background}, Sect.~\ref{sec:algorithm} reduces partial identification to a linear programme, which Sect.~\ref{sec:general} extends to accommodate conditional queries, nested counterfactuals, and structural constraints. Sect.~\ref{sec:examples} demonstrates practical case studies, followed by conclusions in Sect.~\ref{sec:conclusions}; proofs, additional experiments and discussions are in Apps.~\ref{sec:appendix_A},~\ref{sec:appendix_B} and~\ref{sec:appendix_C}, respectively.

\section{Background}\label{sec:background}
\paragraph{Notation.} Let us first review some necessary background on causality. We denote variables by capital letters, whilst small and calligraphic letters are used instead for the states and the sample spaces. Thus, $x\in\mathcal{X}$ is a state of $X$. We consider discrete variables only. Bold is used for sets of variables. The indicator function that is one when its argument is true and zero otherwise is denoted as $\llbracket \cdot \rrbracket$, whilst $|\cdot|$ is the cardinality of the set in the argument.

\paragraph{Structural Causal Models (SCMs).}
Following \citet{pearlCausalityModelsReasoning2009}, an SCM $\mathcal{M} \coloneq \langle \bm{V}, \bm{U}, \mathcal{F}, $ $P(\bm{U})\rangle$ is such that $\bm{V}$ is its set of endogenous variables, $\bm{U}$ the set of exogenous variables distributed according to $P(\bm{U})$, and $\mathcal{F}$ a set $\{f_V\}_{V\in \bm{V}}$ of \emph{structural equations} (SEs). The SE of each $V\in\bm{V}$ determines $v=f_V(\mathrm{pa}_V, \bm{u}_V)$, where $\mathrm{pa}_V \in \mathrm{Pa}_V \subseteq \bm{V} \setminus \{V\}$ and $\bm{U}_V \subseteq \bm{U}$ are, respectively, the endogenous and the exogenous inputs of the SE. We may view an SE as a collection of deterministic mechanisms, or mappings, from input configurations to output states. If $P(\bm{U})$ is unavailable, $M\coloneq\langle \bm{V}, \bm{U}, \mathcal{F} \rangle$ is termed \emph{partially specified} SCM (PSCM). 

Both SCM $\mathcal{M}$ and PSCM $M$ induce a directed graph $\mathcal{G}$ whose nodes represent variables and an edge $(X,Y)$ exists iff $X \in \mathrm{Pa}_Y$. We restrict our analysis to \emph{recursive} SCMs, i.e., such that $\mathcal{G}$ is \emph{acyclic}. We adopt a kinship terminology (e.g., \emph{parents}) to present graphical relations between variables. The notation $\mathbb{M}_M$ is finally used for the set of all the SCMs that share the SEs' signature, i.e., parent and child variables, with $M$.

\paragraph{Interventions.} In an SCM $\mathcal{M}$, $P(\bm{U})$ and $\mathcal{F}$ induce a joint \emph{observational} distribution $P_{\mathcal{M}}(\bm{V})$. An SCM \emph{natural} regime can be modified by acting on its SEs. For an arbitrary set $\bm{X} \subseteq \bm{V}$, an \emph{intervention} yields a new model $\mathcal{M}_{\bm{x}}$ (i.e., a \emph{sub-model}) by replacing the SEs of $\bm{X}$ with constant assignments $\bm{X}\gets\bm{x}$, whilst keeping $P(\bm{U})$ and all other SEs unchanged. The resulting, \emph{interventional}, distribution is denoted as $P_{\mathcal{M}_{\bm{x}}}(\bm{V} \setminus \bm{X})$. For any $Y\in\bm{V}$, the \emph{potential response} $Y_{\bm{x}}(\bm{u})$ denotes the value of $Y$ in $\mathcal{M}_{\bm{x}}$ given $\bm{U}=\bm{u}$. A \emph{counterfactual} variable (or \emph{potential outcome}) $Y_{\bm{x}}$ is the variable induced by $Y_{\bm{x}}(\bm{u})$ when $\bm{u} \sim P(\bm{U})$. 

\paragraph{Counterfactuals.}
Following the logical formalisation of a PSCM by \citet{halpernAxiomatizingCausalReasoning2000}, we define an \emph{atomic proposition} as the assignment $Y_{\bm{x}}=y$ (for short, $y_{\bm{x}}$), where $Y \in \bm{V}, \bm{X}\subseteq \bm{V}$ and $y \in \mathcal{Y}$. A \emph{counterfactual event} $\gamma$ is a conjunction of $k$ atomic propositions, $\gamma \equiv y^{(1)}_{\bm{x}^{(1)}} \land \dots \land y^{(k)}_{\bm{x}^{(k)}}$. A \emph{counterfactual query} $P_\mathcal{M}(\gamma)$ asks for the probability of $\gamma$ with respect to an SCM $\mathcal{M}$. This requires the simultaneous computation of potential responses of multiple sub-models $\{\mathcal{M}_{\bm{x}^{(i)}}\}_{i=1}^k$, coupled by the shared $P(\bm{U})$, i.e.,
\begin{equation}
\label{eq:query_eval}
    P_{\mathcal{M}}(\gamma) = \sum_{\bm{u} \in \bm{\mathcal{U}}} \left\llbracket \bigwedge_{i=1}^k Y^{(i)}_{\bm{x}^{(i)}}(\bm{u}) = y^{(i)} \right\rrbracket \cdot P(\bm{u}) \,.
\end{equation}
The evaluation accumulates over the exogenous instantiations $\bm{u}$ satisfying the formula $\gamma$. We denote such an \emph{entailment} as $(\mathcal{M}, \bm{u}) \models \gamma$. When it is clear from the context which SCM we refer to, we omit specifying it and rewrite the right-hand side of Eq.~\eqref{eq:query_eval} as $\sum_{\bm{u}} \llbracket \bm{u} \models \gamma\rrbracket P(\bm{u})$. If all the subscripts agree on $\bm{x}$, then Eq.~\eqref{eq:query_eval} can be retrieved from the interventional distribution $P_{\mathcal{M}_{\bm{x}}}$. In general, this is not the case and we eventually characterise a \emph{counterfactual} distribution. Finally, we denote as $\bm{V}_\gamma \subseteq \bm{V}$ the endogenous, no matter whether queried or intervened, variables in $\gamma$.

\paragraph{Nested Counterfactuals.}
We can further generalise the intervention for the counterfactual variable $Y_x$ to allow representing settings where the variable $X$ is set to behave as another variable, say $X_z$. A random variable $Y$ in such system is represented with a potential outcome of the form $Y_{X_z}=y$, which is called a \emph{nested counterfactual}. This expresses the statement ``$Y$ becomes $y$ when we set $X$ to whatever value it would have taken if we had set $Z$ to $z$''. Queries involving nested counterfactuals can be rewritten into a marginalisation over conjunctions of standard, atomic counterfactual events \citep{correaNestedCounterfactualIdentification2021}, i.e.,
\begin{equation}\label{eq:cut}
\left(Y_{X_z}=y\right) \equiv \bigvee_{x \in \mathcal{X}} \left( Y_x = y \land X_{z} = x \right) \,.
\end{equation}

\paragraph{(Partial) Identifiability.}
Eq.~\eqref{eq:query_eval} allows to compute counterfactual queries in an SCM $\mathcal{M}$. Yet, the exogenous distribution $P(\bm{U})$ is rarely available, and one must cope with the corresponding PSCM $M$ and a dataset $\mathcal{D}$ of endogenous observations, used to estimate an \emph{empirical} distribution $\hat{P}(\bm{V})$. In general, there are multiple SCMs in $\mathbb{M}_M$, that could have generated $\hat{P}(\bm{V})$ and yield different values to a query over $\gamma$. In such a \emph{partially identifiable} setting, we can only compute bounds to the query, i.e., the minimum:
\begin{equation}
\label{eq:partial_identification_opt_general} 
    \underset{\mathcal{M} \in \mathbb{M}_M}{\min} \,\, P_{\mathcal{M}}(\gamma) \quad \text{s.t.} \quad P_{\mathcal{M}}(\bm{V}) = \hat{P}(\bm{V}) \,,
\end{equation}
and the maximum, which is analogously defined. In the remainder of the paper, just for the sake of brevity, we focus on the minimisation problem, noticing that the maximisation counterpart can be handled analogously. 

Those optimisation tasks have been shown to be NP-hard \citep{zaffalonEfficientComputationCounterfactual2023}. Nevertheless, approximate techniques were proposed to estimate bounds efficiently \citep{zhangPartialCounterfactualIdentification2021,zaffalonEfficientComputationCounterfactual2023,duarte2024automated, bjoru2025DivideConquerCausalComp}. Standard (point) \emph{identifiability} emerges as a special case when the two bounds coincide.

\paragraph{Queries and Partial Orders.}
Although potential outcome semantics allow us to consider a variable $Y_{x}$ for any $Y, X \in \bm{V}$ and $x \in \mathcal{X}$, an intervention on $X$ can alter the realisation of $Y$ only if there exists a directed path from $X$ to $Y$ in the causal graph induced by the assumed SCM $\mathcal{M}$ over $\bm{V}$. This is formalised as follows.
\begin{proposition}[\citeauthor{tian2002studies}, \citeyear{tian2002studies}]\label{prop:topological_invariance}
If $Y, X \in \bm{V}$ are distinct endogenous variables of an SCM $\mathcal{M}$ such that $X \notin \mathrm{Anc}(Y)$, then, $\forall\bm{u}$, $Y_{x}(\bm{u})=Y(\bm{u})$, and thus $P_{\mathcal{M}}(Y_{x}) = P_{\mathcal{M}}(Y)$.
\end{proposition}
Thus, a counterfactual variable $Y_{x}$ can be non-trivial only if $X$ precedes $Y$ in the underlying causal graph. In this paper, we restrict attention to non-trivial queries: consequently, when the graph is not available, a query over $\gamma$ can be also used to induce a collection of precedence constraints, corresponding to a partial order $\prec_\gamma$ on $\bm{V}_\gamma$. A ``cyclic'' event such as $y_x \land x_y$ would be ill-posed as it induces incompatible precedence constraints. Such cases would naturally call for non-recursive models (see, for instance, \citeauthor{cozman25a}, \citeyear{cozman25a} and \citeauthor{bongers2021foundations}, \citeyear{bongers2021foundations}), which fall outside the scope of this paper.

\section{Bounding Framework}\label{sec:algorithm}
The main goal of this paper is to relax the optimisation task in Eq.~\eqref{eq:partial_identification_opt_general} to cope with situations where the underlying PSCM $M$ is unavailable. We start by assuming that the endogenous variables are sorted by a total order $\sigma$, corresponding to $\bm{V} \coloneqq (V^{(1)}, \dots, V^{(k)})$. In this setting, we can replace in Eq.~\eqref{eq:partial_identification_opt_general} the set $\mathbb{M}_{M}$ of SCMs compatible with the PSCM $M$ with the set $\mathbb{M}_{\sigma}$ of \emph{all} SCMs whose causal graph is consistent with $\sigma$, i.e.,
\begin{equation}
\label{eq:partial_identification_opt_fixed_order}
    \underset{\mathcal{M} \in \mathbb{M}_{\sigma}}{\min} P_{\mathcal{M}}(\gamma) \quad\text{s.t.} \quad P_{\mathcal{M}}(\bm{V}) = \hat{P}(\bm{V})\,.
\end{equation}
For the moment, we consider $\gamma$ such that $\bm{V}_\gamma=\bm{V}$, i.e., all endogenous variables appear in the query. A more general account of the case $\bm{V}_\gamma \subset \bm{V}$ is addressed later.

\paragraph{Response Representation.}
Unlike Eq.~\eqref{eq:partial_identification_opt_general}, the optimisation task in Eq.~\eqref{eq:partial_identification_opt_fixed_order} refers to a collection of SCMs possibly based on different PSCMs. Yet, given only $\sigma$ and $\bm{V}$, we consider a \emph{complete canonical} PSCM over $\bm{V}$ denoted as $M_\sigma$ and such that: (i) $\mathrm{Pa}_{V^{(i)}}$ coincides with all the predecessors of $V^{(i)}$ according to $\sigma$ and they are denoted as $\bm{V}_{\prec i}$; (ii) there is a single exogenous variable $U$, which is a parent of all the endogenous variables; (iii) the SEs of each $V \in\bm{V}$ and the cardinality of $U$ are such that $M_\sigma$ is \emph{canonical} in the sense of \citet{zhangPartialCounterfactualIdentification2021}, this basically meaning that the exogenous states of $U$ are indexing \emph{all} possible structural relations between the endogenous variables and their endogenous parents (see App.~\ref{sec:appendix_C} for a more detailed account).

The optimisation task in Eq.~\eqref{eq:partial_identification_opt_fixed_order} can be equivalently achieved by considering the complete canonical PSCM $M_{\sigma}$ only, as stated by the following result.
\begin{theorem}\label{thm:bounding_equivalence}
Eq.~\eqref{eq:partial_identification_opt_general} with $M=M_\sigma$ gives the same minimum of  Eq.~\eqref{eq:partial_identification_opt_fixed_order}, i.e.,
\begin{equation}\label{eq:theorem}
\min_{\mathcal{M} \in \mathbb{M}_{\sigma}} P_{\mathcal{M}}(\gamma) = \min_{\mathcal{M} \in \mathbb{M}_{M_\sigma}} P_{\mathcal{M}}(\gamma)\,,
\end{equation}
where $P_{\mathcal{M}}(\bm{V})=\hat{P}(\bm{V})$ is required on both sides.
\end{theorem}

The above result reduces the optimisation in Eq.~\eqref{eq:partial_identification_opt_fixed_order}---which considers a collection of structurally heterogenous SCMs---to span only those SCMs sharing the same causal graph as $M_\sigma$. This reduction provides a parametrisation of the task, where the optimisation variables are associated with the exogenous probabilities as in Eq.~\eqref{eq:query_eval}. To avoid explicitly modelling these latent variables, we build upon the approach proposed by \citet{balkeCounterfactualProbabilitiesComputational1994, balkeProbabilisticEvaluationCounterfactual1994}. Here, we consider all the potential outcomes that $M_\sigma$ is allowed to generate as to enumerate all possible endogenous mappings from parents to child. This is achieved by considering a \emph{response signature} $\bm{S}_\sigma$ for $M_\sigma$, corresponding to the following collection of potential outcomes:
\begin{equation}\label{eq:signature}
\bm{S}_\sigma \coloneq 
\left\{
V^{(i)}_{\bm{v}_{\prec i}}
\right\}_{\bm{v}_{\prec i}\in \bm{\mathcal{V}}_{\prec i}}^{i=1,\ldots,k}
\,.
\end{equation}
The following result shows that the probability of any counterfactual query can be expressed not only as a linear combination of joint probabilities over the exogenous variables, as in Eq.~\eqref{eq:query_eval}, but also as a linear combination of joint probabilities over $\bm{S}_{\sigma}$.

\begin{theorem}\label{thm:ctf_linearity}
A query $P_{\mathcal{M}}(\gamma)$ in an SCM $\mathcal{M} \in \mathbb{M}_{\sigma}$ can be written as a linear combination of the probabilities of the response signature, i.e.,
\begin{equation}\label{eq:thm_linearity}
P_{\mathcal{M}}(\gamma) = \sum_{\bm{s}_{\sigma} \in \bm{\mathcal{S}}_\sigma} \llbracket \bm{s}_{\sigma} \models \gamma \rrbracket \cdot P(\bm{s}_\sigma) \,.
\end{equation}
\end{theorem}

By viewing observational queries as a special case of counterfactual queries in which interventions are performed on the empty set, we also obtain the following additional result.
\begin{corollary}\label{cor:obs_consistency}
Any observational probability $P_{\mathcal{M}}(\bm{v})$ in an SCM $\mathcal{M} \in \mathbb{M}_{\sigma}$ can be written as a linear combination of the probabilities of the response signature, i.e.,
\begin{equation}\label{eq:cor_obs}
P_{\mathcal{M}}(\bm{v}) = \sum_{\bm{s}_{\sigma} \in \bm{\mathcal{S}}_\sigma} \llbracket \bm{s}_{\sigma} \models \bm{v} \rrbracket \cdot P(\bm{s}_\sigma) \,.
\end{equation}
\end{corollary}

\paragraph{Linear Programming (LP) Formulation.}
Thm.~\ref{thm:ctf_linearity} and Cor.~\ref{cor:obs_consistency} allow us to address the optimisation task in Eq.~\eqref{eq:partial_identification_opt_fixed_order} via the signature joint states. The resulting linearity of the constraints and the objective function allows us to rewrite the bounding task as a LP task:
\begin{equation} 
\label{eq:lp}
\begin{split}
\min \quad
& \sum_{\bm{s}_{\sigma} \in \bm{\mathcal{S}}_\sigma} \llbracket \bm{s}_{\sigma} \models \gamma \rrbracket \cdot q_{\bm{s}_{\sigma}} \\
\text{s.t.} \quad & \sum_{\bm{s}_{\sigma} \in \bm{\mathcal{S}}_\sigma} \llbracket \bm{s}_{\sigma} \models \bm{v} \rrbracket \cdot q_{\bm{s}_{\sigma}} = \hat{P}(\bm{v}) \,, \forall \bm{v}\in\bm{\mathcal{V}} \;; \quad \sum_{\bm{s}_{\sigma} \in \bm{\mathcal{S}}_\sigma} q_{\bm{s}_{\sigma}} = 1 \;; \quad q_{\bm{s}_{\sigma}} \geq 0 \,,\, \forall \bm{s}_{\sigma} \in \bm{\mathcal{S}}_{\sigma} \,.
\end{split}
\end{equation}
with semantics $q_{\bm{s}_{\sigma}}=P(\bm{s}_\sigma)$ for the optimisation variables.

The following result establishes that the above LP task solves the relaxed optimisation over SCMs consistent with $\sigma$ by yielding \emph{tight} identification bounds. 

\begin{theorem}\label{thm:exact_bounds}
The LP task in Eq.~\eqref{eq:lp} gives the solution of Eq.~\eqref{eq:partial_identification_opt_fixed_order}. Furthermore, there exists an SCM $\mathcal{M} \in \mathbb{M}_{\sigma}$ such that $P_{\mathcal{M}}(\gamma)$ coincides with the solution of the LP.
\end{theorem}

\paragraph{Signature State Partitioning.}
The latent space reduction to finite response-function classes \citep{balkeCounterfactualProbabilitiesComputational1994, duarte2024automated} enables us to neatly solve Eq.~\eqref{eq:partial_identification_opt_fixed_order} through Eq.~\eqref{eq:lp}. However, a further reduction can be achieved by observing that the LP in Eq.~\eqref{eq:lp} depends exclusively on the Boolean evaluations of the signatures against the observations $\bm{v}$ and the counterfactual event $\gamma$. Signatures states that evaluate identically against these observations and events can be grouped into equivalence classes $\bm{t}_\sigma \in \bm{\mathcal{T}}_\sigma \coloneq \bm{\mathcal{S}}_\sigma / \sim_\gamma$, where:
\begin{equation}
\label{eq:signature_equivalence}
\bm{s}_\sigma \sim_\gamma \bm{s}'_\sigma \iff \left((\bm{s}_\sigma \models \gamma) \equiv (\bm{s}'_\sigma \models \gamma)\right) \land \left( \exists \bm{v} \in \bm{\mathcal{V}}\,\,\text{s.t.}\,\,(\bm{s}_\sigma \models \bm{v}) \land (\bm{s}'_\sigma \models \bm{v})\right) \,.
\end{equation}
Each class $\bm{t}_\sigma \in \bm{\mathcal{T}}_\sigma$ inherits the binary evaluations from the signatures it includes, hence $\llbracket \bm{t}_\sigma \models \bm{v} \rrbracket$ and $\llbracket \bm{t}_\sigma \models \gamma \rrbracket$. By defining aggregated masses as variables, $q_{\bm{t}_\sigma} \coloneqq \sum_{\bm{s}_\sigma \in \bm{t}_\sigma} q_{\bm{s}_\sigma}$, and applying Thm.~\ref{thm:ctf_linearity} and Cor.~\ref{cor:obs_consistency}, the LP in Eq.~\eqref{eq:lp} reduces to a \emph{lifted} LP task \citep{kerstingRelationalLinearProgramming2017}:
\begin{equation}\label{eq:lp_reduced}
\begin{split}
\min \quad
& \sum_{\bm{t}_\sigma \in \bm{\mathcal{T}}_\sigma} \llbracket \bm{t}_\sigma \models \gamma \rrbracket \cdot q_{\bm{t}_\sigma} \\
\text{s.t.} \quad & \sum_{\bm{t}_\sigma \in \bm{\mathcal{T}}_\sigma} \llbracket \bm{t}_\sigma \models \bm{v} \rrbracket \cdot q_{\bm{t}_\sigma} = \hat{P}(\bm{v})\,, \forall \bm{v}\in\bm{\mathcal{V}} \;; \quad \sum_{\bm{t}_\sigma \in \bm{\mathcal{T}}_\sigma} q_{\bm{t}_\sigma} = 1 \;; \quad q_{\bm{t}_\sigma} \geq 0\,, \forall \bm{t}_\sigma \in \bm{\mathcal{T}}_\sigma\,.
\end{split}
\end{equation}
The equivalence mapping improves the tractability of the LP in Eq.~\eqref{eq:lp} whilst maintaining both soundness and tightness, as it follows from the following result.
\begin{theorem}
\label{thm:reduced_lp_equivalence}
The lifted LP task in Eq.~\eqref{eq:lp_reduced} gives the solution of Eq.~\eqref{eq:partial_identification_opt_fixed_order}. Furthermore, there exists an SCM $\mathcal{M} \in \mathbb{M}_{\sigma}$ such that $P_{\mathcal{M}}(\gamma)$ coincides with the solution of the LP.
\end{theorem}

\paragraph{Marginalising Non-Queried Variables.}
So far, we assumed that all the endogenous variables to appear in the query, i.e., $\bm{V}_\gamma=\bm{V}$. We now extend our results to the more general case in which $\bm{V}_\gamma\subseteq\bm{V}$. To this end, let $\mathbb{M}_{\sigma}^{\bm{V}_\gamma}$ denote the set of SCMs defined only over $\bm{V}_\gamma$ and compatible with $\sigma$. Observe that $\mathbb{M}_{\sigma}^{\bm{V}}=\mathbb{M}_{\sigma}$ and, in general, $\mathbb{M}_{\sigma}^{\bm{V}_\gamma}\subseteq \mathbb{M}_{\sigma}$, where, for notational ease, we assumed that the endogenous variables might not appear explicitly in an SCM, thus being uniformly distributed. The following result holds.
\begin{theorem}\label{thm:opt_equivalence_marginal}
The solution of Eq.~\eqref{eq:partial_identification_opt_fixed_order} coincides with:
\begin{equation}\label{eq:partial_identification_opt_fixed_order_marginal}
    \underset{\mathcal{M} \in \mathbb{M}_{\sigma}^{\bm{V}_\gamma}}{\min} P_{\mathcal{M}}(\gamma) \quad\text{s.t.} \quad P_{\mathcal{M}}(\bm{V}_\gamma) = \hat{P}(\bm{V}_\gamma)\,.
\end{equation}
where $\hat{P}(\bm{V}_\gamma)$ is obtained from $\hat{P}(\bm{V})$ by marginalisation. 
\end{theorem}

As a consequence of the above result, when $\bm{V}_\gamma\subset\bm{V}$, we marginalise the empirical distribution before applying the LP formulation to solve Eq.~\eqref{eq:partial_identification_opt_fixed_order_marginal} and hence Eq.~\eqref{eq:partial_identification_opt_fixed_order}. 

\paragraph{Coping with Partial Orders.} The optimisation task in Eq.~\eqref{eq:partial_identification_opt_fixed_order} assumes a total order $\sigma$ over the endogenous variables. However, as discussed in Sect.~\ref{sec:background}, a counterfactual query over $\gamma$ might induce only a partial order $\prec_\gamma$. Nevertheless, a general formulation is readily obtained by considering the bound over all linear extensions $\Sigma(\prec_\gamma)$:
\begin{equation}
\label{eq:partial_identification_over_orders}
    \underset{\sigma \in \Sigma(\prec_\gamma)}{\min} \, \underset{\mathcal{M} \in \mathbb{M}^{\bm{V}_\gamma}_{\sigma}}{\min} P_{\mathcal{M}}(\gamma) \quad\text{s.t.} \quad P_{\mathcal{M}}(\bm{V}_\gamma) = \hat{P}(\bm{V}_\gamma) \,.
\end{equation}
When $\gamma$ contains no interventions, $\prec_\gamma$ imposes no ordering constraints, so that all possible total orders must be considered, and we have $|\Sigma(\prec_\gamma)| = |\bm{V}_\gamma|!$. Interventions restrict this search space by ruling out inconsistent orderings; nonetheless, the number of linear extensions can still scale exponentially with respect to $|\bm{V}_\gamma|$. In spite of this, the following result establishes that evaluating an arbitrary linear extension is sufficient, thus eliminating the need for a brute-force optimisation over all the orders in $\Sigma(\prec_\gamma)$.

\begin{theorem}
\label{thm:order_invariance}
Consider a counterfactual query over $\gamma$ inducing a partial order $\prec_\gamma$. Any total order $\sigma \in \Sigma(\prec_\gamma)$ gives the same minimum in Eq.~\eqref{eq:partial_identification_opt_fixed_order_marginal} and, hence, the solution of Eq.~\eqref{eq:partial_identification_over_orders}.
\end{theorem}
As a consequence of Thm.~\ref{thm:order_invariance}, we are not required to assume or enforce any specific total order. The choice of a particular $\sigma$ serves only as a possible parametrisation of the task.

\paragraph{Complexity.}
Thm.~\ref{thm:exact_bounds} enables us to solve the optimisation problem in Eq.~\eqref{eq:partial_identification_opt_fixed_order} using a linear solver that implements the LP defined in Eq.~\eqref{eq:lp}. It is straightforward to verify that, if the number of variables in the query is $k \coloneq |\bm{V}_{\gamma}|$ (see Thm.~\ref{thm:opt_equivalence_marginal}) and $n \coloneq \max_{i=1}^k |\mathcal{V}^{(i)}|$ denotes the maximum endogenous cardinality, the number of potential outcomes comprising the signature $\bm{S}_\sigma$ in Eq.~\eqref{eq:signature} is bounded by $\sum_{i=0}^{k-1} n^i = \frac{n^k-1}{n-1}$. The number of joint signature states $|\bm{\mathcal{S}}_\sigma|$, which dictates the number of optimisation variables in the na\"ive LP of Eq.~\eqref{eq:lp}, is therefore double exponential, namely $\mathcal{O}(n^{n^{k-1}})$. In contrast, by Thm.~\ref{thm:reduced_lp_equivalence}, the size of the quotient space $|\bm{\mathcal{T}}_\sigma|$, which governs the size of the lifted LP in Eq.~\eqref{eq:lp_reduced}, is bounded by $\mathcal{O}\big(\prod_{i=1}^k |\mathcal{V}^{(i)}|\big) = \mathcal{O}(n^k)$. Thus, to partially identify any query given the empirical marginal $\hat{P}(\bm{V}_\gamma)$, we only need to solve Eq.~\eqref{eq:lp_reduced}, whose optimisation space scales linearly with the size of the sample space of $\bm{V}_\gamma$.

\section{Extensions}\label{sec:general}
In the current section, we present some important extensions of the LP mapping introduced in the previous section.

\paragraph{Conditional Queries.}
To bound a conditional query $P(\gamma \mid \delta) \coloneq P(\gamma \wedge \delta)/P(\delta)$, it suffices to note that both terms decompose linearly over $\bm{\mathcal{S}}_\sigma$ via Thm.~\ref{thm:ctf_linearity} and Cor.~\ref{cor:obs_consistency}. This reformulates the bounding task as a \emph{linear-fractional} programme, which is reduced to a standard LP via the Charnes-Cooper transformation \citep{charnesProgrammingLinearFractional1962}. We incorporate evidence by refining the equivalence relation in Eq.~\eqref{eq:signature_equivalence} to evaluate both $\gamma \land \delta$ and $\delta$:
\begin{equation}
    \bm{s}_\sigma \sim_{\gamma \mid \delta} \bm{s}'_\sigma \iff \bm{s}_\sigma \sim_{\gamma\land\delta} \bm{s}'_\sigma \land \left( (\bm{s}_\sigma \models \delta) \equiv (\bm{s}'_\sigma \models \delta) \right) \,.
\end{equation} 
Letting $b_{\psi} \coloneq \llbracket \bm{s}_\sigma \models \psi \rrbracket$, the logical constraint $\gamma \land \delta \Rightarrow \delta$ enforces $b_{\gamma \land \delta} \leq b_{\delta} \in \{0,1\}$. This restricts evaluations to three valid Boolean states: $(b_{\gamma\land\delta}, b_\delta) \in \{(0,0), (0,1), (1,1)\}$. The optimisation space retains the same asymptotic size discussed in Sect.~\ref{sec:algorithm}.

\paragraph{Nested Queries.}
The LP mapping in Eq.~\eqref{eq:lp}, and by extension the lifted LP in Eq.~\eqref{eq:lp_reduced}, accommodates nested counterfactuals. Although nested variables do not appear directly in the response signature $\bm{S}_\sigma$, their evaluation remains a linear optimisation task. For a single nested counterfactual, it is sufficient to consider the probabilistic version of Eq.~\eqref{eq:cut}, i.e.,
\begin{equation}\label{eq:unnest}
P_{\mathcal{M}}(Y_{X_z}=y)= \sum_{x\in \mathcal{X}} P_{\mathcal{M}}(Y_x=y,X_z=x)\,.
\end{equation}
In the general case, it is sufficient to perform such un-nesting on all the nested terms to obtain a linear function of standard counterfactual queries and thus a linear objective function.

\paragraph{Experimental Data.}
Beyond observational data, it is straightforward to incorporate experimental distributions obtained directly from randomised controlled trials. By Thm.~\ref{thm:ctf_linearity}, experimental probabilities translate directly into linear constraints on the signature probabilities in Eq.~\eqref{eq:lp}. To preserve the validity of the reduction in Eq.~\eqref{eq:lp_reduced}, the equivalence relation is refined by further requiring that equivalent signatures agree on the Boolean evaluation of the relevant interventional events. For instance, given an experimental probability $P(y_x)$, the relation in Eq.~\eqref{eq:signature_equivalence} is augmented with the logical condition $(\bm{s}_\sigma \models y_x) \equiv (\bm{s}'_\sigma \models y_x)$. Tracking these additional constraints partitions the signature space more finely, thereby scaling the quotient space by a factor determined by the number of interventional events considered. Furthermore, when multiple studies or heterogeneous data sources are available, one can seamlessly incorporate all evidence by appending the constraints induced by each trial in the same LP. Such a procedure natively performs automated data fusion, analogously to the approach developed by \citet{zaffalon2023} for PSCMs.

\paragraph{Monotonicity and Weak Exogeneity.}
For ordinal variables, \emph{monotonicity} \citep{Angrist1996IdentificationCausalEffectsIV, manskiMonotoneTreatmentResponse1997} posits a directional effect (e.g., $Y_{x_1} \geq Y_{x_0}$ if $x_1>x_0$). \emph{Weak exogeneity} \citep{tianProbabilitiesCausationBounds2000} requires that a variable is unaffected by unobserved confounders. Both assumptions translate expert knowledge into logical restrictions on the response signature $\bm{\mathcal{S}}_\sigma$: monotonicity excludes states violating the inequality by fixing relevant optimisation variables to zero; exogeneity requires that causal effects are identified from data, i.e., $P(y_{x})=P(y\mid x)$, functionally mirroring the embedding of experimental data.

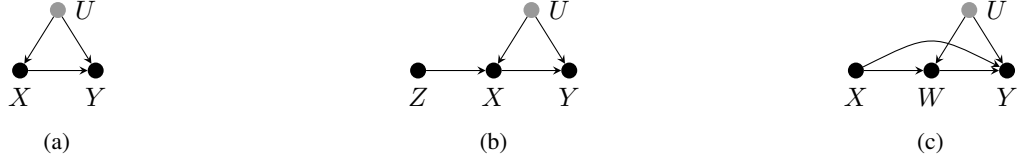
\begin{figure}[!t]
  \centering
  \begin{subfigure}[b]{0.3\textwidth}
    \centering
    \begin{tikzpicture}[scale=1.0]
        \node[dot,label=below:{$X$}] (x)  at (1,0) {};
        \node[dot,label=below:{$Y$}] (y)  at (2,0) {};
        \node[dot2,label=right:{$U$}] (u)  at (1.5,0.8) {};
        \draw[a2] (x) -- (y);
        \draw[a2] (u) -- (x);
        \draw[a2] (u) -- (y);
    \end{tikzpicture}
    \caption{}
    \label{fig:2}
  \end{subfigure}%
  \hfill
  \begin{subfigure}[b]{0.3\textwidth}
    \centering
    \begin{tikzpicture}[scale=1.0]
        \node[dot,label=below:{$Z$}] (z)  at (0,0) {};
        \node[dot,label=below:{$X$}] (x)  at (1,0) {};
        \node[dot,label=below:{$Y$}] (y)  at (2,0) {};
        \node[dot2,label=right:{$U$}] (u)  at (1.5,0.8) {};
        \draw[a2] (z) -- (x);
        \draw[a2] (x) -- (y);
        \draw[a2] (u) -- (y);
        \draw[a2] (u) -- (x);
    \end{tikzpicture}
    \caption{}
    \label{fig:balke}
  \end{subfigure}%
  \hfill
  \begin{subfigure}[b]{0.3\textwidth}
    \centering
    \begin{tikzpicture}[scale=1.0]
        \node[dot,label=below:{$X$}] (x)  at (0,0) {};
        \node[dot,label=below:{$W$}] (w)  at (1,0) {};
        \node[dot,label=below:{$Y$}] (y)  at (2,0) {};
        \node[dot2,label=right:{$U$}] (u)  at (1.5,0.8) {};
        \draw[a2] (x) -- (w);
        \draw[a2] (w) -- (y);
        \draw[a2] (u) -- (y);
        \draw[a2] (u) -- (w);
        \draw[a2] (x) .. controls (1,0.5) .. (y);
    \end{tikzpicture}
    \caption{}
    \label{fig:mediator}
  \end{subfigure}
  \caption{{Causal diagrams considered in the examples.} (a)~Confounded treatment-outcome pair. (b)~Instrumental Variable (IV) setting, where $Z$ serves as an instrument for $X$. (c)~Mediation model with a direct effect and confounding between the mediator $W$ and the outcome. Grey nodes denote unobserved exogenous variables.}
  \label{fig:dags}
\end{figure}

\section{Examples}\label{sec:examples}
Let us validate the application of our LP-based bounding scheme across a range of scenarios. We compare our bounds against analytical formulae and results from the literature.\footnote{A Python implementation of our method is available at \url{www.github.com/Erhtric/copid}. The experimental setup employed the \texttt{COIN-OR CBC} open-source solver via the PuLP modeller (\url{www.coin-or.org}).}

\paragraph{Probabilities of Causation.}
As a first illustrative application of our method, we derive bounds for the well-known  
\emph{probability of necessity and sufficiency} (PNS, \citeauthor{pearlProbabilitiesCausationThree1999}, \citeyear{pearlProbabilitiesCausationThree1999}). This query refers to the counterfactual event $\gamma \coloneq (Y_{X=0}=0) \land (Y_{X=1}=1)$, which is intended to capture the causal relationship between two Boolean endogenous variables, $X$ and $Y$. The only total order $\sigma$ compatible with $\prec_\gamma$ is $\sigma:=(X,Y)$. The associated response signature is therefore $\bm{S}_\sigma = \{X, Y_{X=0}, Y_{X=1}\}$. Accordingly, PNS bounds are obtained by solving LPs with eight variables, corresponding to the joint states of $\bm{S}_\sigma$; and four linear constraints induced by the observational distribution $\hat{P}(X,Y)$. The resulting solutions coincide with the classical closed-form bounds derived by \citet{tianProbabilitiesCausationBounds2000}. The equivalence relation in Eq.~\eqref{eq:signature_equivalence} prunes the parameter space by identifying empty equivalence classes. Specifically, the candidate classes defined by the Boolean tuples $\langle (X=x, Y=y), \llbracket \gamma \rrbracket\rangle$ for $(x,y) \in \{(1,0), (0,1)\}$ contain no valid signatures. This is because an observational state of $(X=1, Y=0)$ forces $Y_{X=1}=0$ via consistency \citep{gallesAxiomaticCharacterizationCausal1998}, directly contradicting the requirement in $\gamma$ that $Y_{X=1}=1$. We similarly proceed for $(X=0, Y=1)$. 

The \emph{tightness} (referred to as \emph{sharpness} in the original work) of these bounds follows from Thm.~\ref{thm:exact_bounds}: one can construct two complete canonical SCMs under which the query attains the lower and upper bounds (see also the constructive argument in the proof).

Finally, by considering the extension to conditional queries discussed in Sect.~\ref{sec:general}, analogous results can be readily obtained for the other \emph{probabilities of causation} (PN and PS, \citeauthor{pearlProbabilitiesCausationThree1999}, \citeyear{pearlProbabilitiesCausationThree1999}). Moreover, unlike the analytical results of \citet{tianProbabilitiesCausationBounds2000}, which are restricted to binary variables, our approach naturally generalises to the non-binary case. By simply adjusting the cardinality of the variables in the response signature $\bm{S}_\sigma$, we can derive tight bounds for probabilities of causation involving multi-valued treatments and outcomes \citep{sun2025bounding} without requiring new algebraic derivations. From this perspective, our work is a genuine generalisation of the analysis in \citet{tianProbabilitiesCausationBounds2000}, with a clear semantics, and allowing to bound arbitrary counterfactual queries, whether nested or not.

\paragraph{Average Causal Effect (ACE).}
The ACE measures the average change due to treatment: 
\begin{equation}\label{eq:ace}
\mathrm{ACE}_{x,x'}(y) \coloneq P(Y_{x}=y) - P(Y_{x'}=y) \,.
\end{equation}
For an outcome $Y \in \{y,y'\}$ and treatment $|\mathcal{X}|=3$, the unique order $\sigma=(X,Y)$ induces an LP with 24 signature variables and six empirical constraints in Eq.\eqref{eq:lp}. The resulting bounds, $[ P(x,y) + P(x',y') - 1 , 1 - P(x,y') - P(x',y) ]$ \citep{balkeBoundsTreatmentEffects1997}, recover those of \citet[Sect.~6.1]{sachsGeneralMethodDeriving2023} whose causal assumptions are in Fig.~\ref{fig:2}. Crucially, our method achieves this without explicitly parametrising a fully specified causal graph nor the SEs. A different account of the ACE bounding can be found in App.~\ref{sec:appendix_B}.

\paragraph{Controlled Direct Effect (CDE).}
\begin{table}[htp!]
\footnotesize
\centering
\setlength{\tabcolsep}{4pt}
\begin{tabular}{@{}lrrrrrr@{}}
\toprule
\multicolumn{1}{c}{\multirow{2}{*}{Assumption Set}} & \multicolumn{2}{c}{$\mathrm{CDE}^{w_0}_{x_0,x_1}(y_1)$} & \multicolumn{2}{c}{$\mathrm{CDE}^{w_1}_{x_0,x_1}(y_1)$} & \multicolumn{1}{c}{\multirow{2}{*}{$|\bm{\mathcal{S}}_\sigma|$}} & \multicolumn{1}{c}{\multirow{2}{*}{$|\bm{\mathcal{T}}_\sigma|$}} \\ \cmidrule(lr){2-3} \cmidrule(lr){4-5}
& Lower & Upper & Lower & Upper & & \\ \midrule
\textit{No additional assumptions} & \multicolumn{6}{l}{}\\
\hspace{3mm}\citet{caiBoundsDirectEffects2008} (Fig.~\ref{fig:mediator}) & -0.200 & 0.385 & -0.781 & 0.634 & -- & -- \\
\hspace{3mm}This paper & -0.599 & 0.694 & -0.891 & 0.817 & 128 & 20 \\
\midrule
\textit{Monotonicity} & \multicolumn{6}{l}{}\\
\hspace{3mm}\citet{caiBoundsDirectEffects2008} & 0.000 & 0.036 & 0.000 & 0.582 & -- & -- \\
\hspace{3mm}This paper & 0.000 & 0.519 & 0.000 & 0.791 & 36 & 13 \\
\midrule
\textit{Weak exogeneity} & \multicolumn{6}{l}{}\\
\hspace{3mm}This paper & -0.200 & 0.385 & -0.782 & 0.633 & 128 & 64 \\
\midrule
\textit{Weak exogeneity + monotonicity} & \multicolumn{6}{l}{}\\
\hspace{3mm}This paper & 0.000 & 0.035 & 0.000 & 0.580 & 36 & 24 \\
\bottomrule
\end{tabular}
\caption{Comparison of CDE bounds on the (binary) Lipid data \citep[Tab.~1]{caiBoundsDirectEffects2008} against analytical bounds under varying assumptions. In this experiment, \emph{Monotonicity} forces positive effects of $X$ on $W$, $X$ on $Y$, and $W$ on $Y$; whilst, \emph{weak exogeneity} imposes no-interaction between $X$ and $W$ on $Y$.
}
\label{tab:cde_results}
\end{table}

Consider a setting in which the effect of a treatment $X$ on an outcome $Y$ is \emph{mediated} by a variable $W$. In this context, the ACE in Eq.~\eqref{eq:ace} can be refined into a CDE to measure the change due to treatment at different mediator \emph{strata}:
\begin{equation}\label{eq:cde}
\mathrm{CDE}^w_{x_0,x_1}(y)\coloneq P(Y_{x_1,w}=y)-P(Y_{x_0,w}=y)\,.
\end{equation}
The query induces a partial order, namely $\{X, W\} \prec Y$. By Thm.~\ref{thm:order_invariance}, we can choose an arbitrary order among its set of linear extensions, together with an empirical distribution $\hat{P}(X,W,Y)$, to solve Eq.~\eqref{eq:lp_reduced}. We validate the method against the problem proposed by \citet{caiBoundsDirectEffects2008} of bounding Eq.~\eqref{eq:cde} under the assumptions encoded in Fig.~\ref{fig:mediator}. In such a setting, we first consider all binary variables and a randomised treatment (i.e., no confounders between $X$ and $Y$), examining data from the \emph{Coronary Primary Prevention Trial} \citep{caiBoundsDirectEffects2008}. 

A summary of the computed bounds is reported in Tab.~\ref{tab:cde_results}. One can observe that when no assumptions are made, or when solely monotonicity is assumed, the minimal structure inherently implied by the query yields valid, but wide, intervals. Notably, the tight bounds reported by \citet{caiBoundsDirectEffects2008} rely on a fully specified causal graph. To mirror their structural assumptions within our framework, one must incorporate a weak exogeneity constraint---also referred to as the \emph{no-interaction} assumption---into the signature space $\bm{\mathcal{S}}_\sigma$. Specifically, enforcing the independence $\{Y,W\} \perp\!\!\!\perp X$ recovers \citet{caiBoundsDirectEffects2008} analytical bounds. As a special instance, one can also require $Y \perp\!\!\!\perp \{X, W\}$ to collapse the bound to a point estimate, since $P(y_{x,w})=P(y\mid x,w)$. As expected, introducing these independence constraints expands the number of optimisation parameters $|\bm{\mathcal{T}}_\sigma|$, whilst introducing monotonicity constraints shrinks the space (cf.~Tab.~\ref{tab:cde_results}). A similar discussion holds for categorical variables, where the LP reduction can be systematically constructed and compared against the equations presented in \citet{caiBoundsDirectEffects2008}.

\paragraph{Natural Direct Effect (NDE).}
\begin{table}[htp!]
\footnotesize
\centering
\begin{tabular}{lrrrr}
\toprule
 & Lower & Upper & $|\bm{\mathcal{S}}_\sigma|$ & $|\bm{\mathcal{T}}_\sigma|$ \\
\midrule
\multicolumn{5}{l}{\textit{Assumptions} for $\mathrm{NDE}_{x_0,x_1}(y_1)$} \\
\hspace{3mm}Query-induced Order & -0.691 & 0.688 & 294,912 & 60 \\
\hspace{3mm}$Y \perp\!\!\!\perp X$ & -0.503 & 0.497 & 294,912 & 96 \\
\hspace{3mm}$\{Y, W\} \perp\!\!\!\perp X$ & -0.503 & 0.497 & 294,912 & 528 \\
\hspace{3mm}$Y \perp\!\!\!\perp \{X, W\}$ & -0.628 & 0.679 & 294,912 & 89,172 \\
\hspace{3mm}$Y \perp\!\!\!\perp \{X, W\}$ \& $W \perp\!\!\!\perp X$ & -0.489 & 0.637 & 294,912 & 294,912 \\
\midrule
\midrule
\multicolumn{5}{l}{\textit{Query}} \\
\hspace{3mm}$P(Y_{x_1,W=w_0}=y_1,W_{x_0}=w_0)$ & 0.000 & 0.610 & 294,912 & 37 \\
\hspace{3mm}$P(Y_{x_1,W=w_1}=y_1,W_{x_0}=w_1)$ & 0.000 & 0.493 & 294,912 & 37 \\
\hspace{3mm}$P(Y_{x_1,W=w_2}=y_1,W_{x_0}=w_2)$ & 0.000 & 0.365 & 294,912 & 37 \\
\hspace{3mm}$P(Y_{x_1,W=w_3}=y_1,W_{x_0}=w_3)$ & 0.000 & 0.415 & 294,912 & 37 \\
\hspace{3mm}$P(Y_{x_1,W=w_4}=y_1,W_{x_0}=w_4)$ & 0.000 & 0.357 & 294,912 & 37 \\
\hspace{3mm}$P(Y_{x_1,W=w_5}=y_1,W_{x_0}=w_5)$ & 0.000 & 0.391 & 294,912 & 37 \\
\hspace{3mm}$P(Y_{x_0}=y_1)$                   & 0.312 & 0.691 & 8 & 6 \\
\bottomrule
\end{tabular}
\caption{UC Berkeley dataset: NDE bounds under various weak exogeneity assumptions (top) and individual unnested query terms under query-induced assumptions (bottom).}
\label{tab:berkeley}
\end{table}
The CDE in Eq.~\eqref{eq:cde} refers to a fixed state of the mediator. By contrast, the NDE evaluates the effect of a treatment when the mediator is set to the value it would have attained in the absence of treatment. This leads to the nested counterfactual query:
\begin{equation}\label{eq:nde}
\mathrm{NDE}_{x_0,x_1}(y_1) \coloneqq P(Y_{x_1, W_{x_0}}=y_1)-P(Y_{x_0}=y_1)\,.
\end{equation}
As a matter of fact, the nested syntactical structure of the query induces a unique complete order $X \prec W \prec Y$. We proceed to apply the un-nesting procedure in Eq.~\eqref{eq:unnest}, which, in turn, allows one to rewrite nested counterfactuals like in Eq.~\eqref{eq:nde} as sums of standard (non-nested) counterfactual queries. This transformation preserves the linear structure of the objective function in our LP formulation and one can consider the PSCM whose endogenous structure is like in Fig.~\ref{fig:mediator} but with one exogenous $U$ parent of all variables. As discussed previously, by Thm.~\ref{thm:reduced_lp_equivalence}, the optimisation space's size is bounded by $O(|\mathcal{X}| \cdot |\mathcal{W}| \cdot |\mathcal{Y}|)$ (cf.~Tab.~\ref{tab:berkeley}), making our approach computationally feasible for modern LP solvers. To the best of our knowledge, no general bounding techniques are currently available for queries of this kind without imposing additional structural assumptions about the underlying causal graph.

\paragraph{Berkeley Admissions Dataset.} For a numerical illustration, we consider the popular \emph{{UC} Berkeley admissions} dataset \citep{bickelSexBiasGraduate1975}. In this dataset, the binary variables $X$ and $Y$ represent \emph{gender} (with $x_1$ indicating \emph{female}) and \emph{admission outcome} (with $y_1$ denoting \emph{acceptance}). The mediator $W$ is a categorical variable with six states corresponding to the department to which the applicant applied. We omit monotonicity assumptions, as deterministic effects are tricky to justify in social contexts. 

Tab.~\ref{tab:berkeley} summarises the total NDE bounds under varying weak exogeneity assumptions, alongside the decomposed counterfactual terms from Eq.~\eqref{eq:unnest}. In the context of university admissions, enforcing $Y \perp\!\!\!\perp X$ assumes no confounding jointly affecting gender and admission. The stricter condition $\{Y, W\} \perp\!\!\!\perp X$ treats gender as if it were completely randomised, assuming department choice does not depend on gender. Conversely, $Y \perp\!\!\!\perp \{X, W\}$ dictates that the admission committee's decision is entirely independent of unobserved background factors influencing either the applicant's gender or department  preference. It is worth noting that when all possible exogeneity assumptions are employed, the equivalence reduction from Sect.~\ref{sec:background} fails to compress the optimisation space, as enforcing mutual independence across all response mechanisms requires tracking each signature state individually. Whilst these assumptions narrow the bounds, at the cost of expanding the optimisation space $|\bm{\mathcal{T}}_\sigma|$, they risk misrepresenting complex dynamics.

Unlike bespoke analytical methods that require graphical constraints to bound nested counterfactuals, our LP formulation automatically evaluates the NDE using only the minimal assumptions inherent to the query itself.
For reference, under the \emph{standard fairness model} \citep{pleckoCausalFairnessAnalysis2024}, the NDE evaluates to a point estimate of 0.043. Concluding a lack of discriminative policies based on this point estimate might be completely inaccurate, as the stark contrast with our wider, assumption-free intervals demonstrates how such conclusions rely heavily on rigid structural assumptions. Furthermore, it is interesting to notice that directly bounding the full NDE yields substantially tighter intervals than combining separately bounded decomposed components via interval arithmetic.

\section{Conclusions}\label{sec:conclusions}
We presented a linear programming framework for partially identifying arbitrary, including nested, counterfactual queries, relying exclusively on structural assumptions inherent to the query itself. In doing so, we successfully extended the classical bounding results of \citet{tianProbabilitiesCausationBounds2000} for probabilities of causation. We proved that partial identification under these minimal assumptions yields tight intervals, and established their invariance with respect to the  marginalisation of the variables not appearing in the query. Rather than requiring an input causal graph, an SCM attaining the bounds emerges constructively as a by-product. Furthermore, exploiting query symmetries drastically reduces the optimisation space, rendering the LP tractable. We acknowledge an inherent trade-off in the proposed method: reducing the reliance on prior domain knowledge, and explicit conditional independencies, produces intervals that are naturally wider than those derived from a fully specified causal graph. Nevertheless, as demonstrated in our case studies, this framework serves as a general and computationally efficient method to bound arbitrary inquiries under minimal assumptions. Future work could investigate extending this framework to cyclic causal structures, examining the impact of different valid causal orders, and investigate finer quotient-space reductions to further enhance scalability.

\newpage
\section*{Acknowledgements}
The research of Eric Rossetto was funded by Swiss Post AG as part of a Doctoral Studies Grant on \emph{AI and Robustness}. The authors thank the anonymous reviewers for their constructive comments, and in particular the first reviewer for insightful remarks on the underlying assumptions of our approach and its possible extension to cyclic models.


\appendix
\section{Proofs}\label{sec:appendix_A}
 \begin{proof}[\textbf{Proof of Thm.~\ref{thm:bounding_equivalence}}]
Let us separately prove the non-strict inequalities between the two sides of Eq.~\eqref{eq:theorem}. By definition, $M_\sigma$ is based on a causal graph compatible with $\sigma$. This implies $\mathbb{M}_{M_\sigma} \supseteq \mathbb{M}_\sigma$ and hence the $\leq$ inequality.

In order to prove the opposite inequality, let $\mathcal{M}^*$ be an SCM that solves Eq.~\eqref{eq:partial_identification_opt_fixed_order}. By definition of $\mathbb{M}_{\sigma}$, the causal graph of $\mathcal{M}^*$ is consistent with the total order $\sigma$. Therefore, for each $V^{(i)} \in \bm{V}$, the endogenous parents of $V^{(i)}$ in $\mathcal{M}^*$ form a subset of its predecessors $\bm{V}_{\prec i}$ under $\sigma$.
We can enlarge each SE of $V^{(i)}$ by including as inputs all variables in $\bm{V}_{\prec i}$ that are not already endogenous parents of $V^{(i)}$, and by similarly incorporating all exogenous variables not already among its exogenous parents. In this way, the collection $\bm{U}$ of exogenous variables can be treated as a single aggregated variable $U$. Consequently, $\mathcal{M}^*$ may be regarded as an SCM that is consistent with the same causal graph as the PSCM $M_{\sigma}$. Since $M_{\sigma}$ is canonical, there exists an equivalent SCM in $\mathbb{M}_{M_{\sigma}}$. This implies the $\geq$ inequality.
\end{proof}

\begin{lemma}
\label{lemma:signature_entailment}
For any atomic counterfactual proposition $V_{\bm{x}}=v$ and $\bm{s}_\sigma \in \bm{\mathcal{S}}_\sigma$, the truth value of the entailment $\bm{s}_\sigma \models (V_{\bm{x}}=v)$ is determined. Furthermore, if $\gamma$ is a conjunction of such propositions, the truth value of $\bm{s}_\sigma \models \gamma$ is also determined.
\end{lemma}

\begin{proof}
Consider an arbitrary intervention $\bm{x}$. We proceed by induction on the topological order $\sigma$ to show that the potential response of every variable $V^{(i)}$ is uniquely determined by $\bm{s}_\sigma$ and $\bm{x}$.

If $i=1$, the variable $V^{(1)}$ has no endogenous parents. If $V^{(1)} \in \bm{X}$, its value is fixed to the assignment specified in $\bm{x}$. If $V^{(1)} \notin \bm{X}$, its value is uniquely determined by the realization encoded in $\bm{s}_\sigma$.

Assume now that for some $i > 1$, the interventional values of all preceding variables $\bm{V}_{\prec i}$ are fixed to $(\bm{v}_{\prec i})_{\bm{x}}$. We determine the value of $V^{(i)}$ as follows:
\begin{enumerate}[noitemsep]
    \item If $V^{(i)} \in \bm{X}$, its value is determined by the assignment in $\bm{x}$.
    \item If $V^{(i)} \notin \bm{X}$, its value is uniquely determined by $\bm{s}_\sigma$ given the values of its parents, which are fixed to $(\bm{v}_{\prec i})_{\bm{x}}$ by the induction hypothesis.
\end{enumerate}

Since $|\bm{V}| < \infty$, every counterfactual variable $V_{\bm{x}}$ evaluates to a unique constant under $\bm{s}_\sigma$. Consequently, the atomic entailment is well-defined. By extension, the entailment for a conjunction $\gamma = \bigwedge_{j=1}^k (Y^{(j)}_{\bm{x}^{(j)}} = y^{(j)})$ is determined by:
\begin{equation}
    (\bm{s}_\sigma \models \gamma) \iff \bigwedge_{j=1}^k \left(\bm{s}_\sigma \models Y^{(j)}_{\bm{x}^{(j)}} = y^{(j)} \right ) \, .
\end{equation}
\end{proof}

\begin{lemma}\label{lemma:exogenous_signature_bijection}
For any arbitrary SCM $\mathcal{M} \in \mathbb{M}_{M_\sigma}$, there exists a bijection $h: \mathcal{U} \to \bm{\mathcal{S}}_\sigma$ between its exogenous domain $\mathcal{U}$ and the response signature domain $\bm{\mathcal{S}}_\sigma$.
\end{lemma}

\begin{proof}
Let $h: \mathcal{U} \to \bm{\mathcal{S}}_\sigma$ map each exogenous state $u$ to its induced signature $\bm{s}_\sigma$. Because SEs are deterministic, each $u$ induces exactly one signature, making $h$ a well-defined function. We establish that $h$ is a \emph{bijection}. 

By definition, $\bm{\mathcal{S}}_\sigma$ comprises all valid configurations of potential outcomes which $M_\sigma$ generates. Because $\mathcal{M} \in \mathbb{M}_{M_\sigma}$, it follows a canonical specification \citep[Def.~2.3]{zhangPartialCounterfactualIdentification2021}; its exogenous domain $\mathcal{U}$ explicitly indexes the Cartesian product of every possible deterministic function mapping from endogenous parents to children. Thus, for any arbitrary signature $\bm{s}_\sigma \in \bm{\mathcal{S}}_\sigma$, there is at least one state $u \in \mathcal{U}$ that generates it. We conclude that $h$ is surjective.

By the same canonical specification of $\mathcal{M}$, the cardinality of $\mathcal{U}$ is exactly equal to the cardinality of the Cartesian product of all possible functional mappings---which coincides with $|\bm{\mathcal{S}}_\sigma|$ by construction in Eq.~\eqref{eq:signature}. Because $h$ is a surjective function mapping between two finite sets of identical cardinality, it must necessarily be injective. Therefore, every state $u$ corresponds uniquely to a signature $\bm{s}_\sigma$.
\end{proof}

\begin{proof}[\textbf{Proof of Thm.~\ref{thm:ctf_linearity}}]
By Thm.~\ref{thm:bounding_equivalence}, without loss of generality, we can restrict our analysis to an arbitrary SCM $\mathcal{M} \in \mathbb{M}_{M_\sigma}$. 

By Lem.~\ref{lemma:exogenous_signature_bijection}, there exists a bijection $h$ mapping each exogenous state $u \in \mathcal{U}$ to a unique response signature $\bm{s}_\sigma \in \bm{\mathcal{S}}_\sigma$. Consequently, the logical equivalence $(u \models \gamma) \iff (\bm{s}_\sigma \models \gamma)$ holds true for $\bm{s}_\sigma=h(u)$. The left-hand side is well-defined by the soundness axiomatisation of SCMs \citep{halpernAxiomatizingCausalReasoning2000}, and the right-hand side is determined by Lem.~\ref{lemma:signature_entailment}. 

Because $h$ is bijective, the probability measure over $\mathcal{U}$ directly translates to $\bm{\mathcal{S}}_\sigma$, yielding $P(U = u) = P(\bm{S}_\sigma = \bm{s}_\sigma)$. Substituting $u$ with $\bm{s}_\sigma$ in Eq.~\eqref{eq:query_eval}, we obtain:
\begin{equation*}
    P_{\mathcal{M}}(\gamma) = \sum_{\bm{s}_\sigma \in \bm{\mathcal{S}}_\sigma} \llbracket \bm{s}_\sigma \models \gamma \rrbracket \cdot P(\bm{S}_\sigma = \bm{s}_\sigma) \,.
\end{equation*}
\end{proof}

\begin{proof}[\textbf{Proof of Thm.~\ref{thm:exact_bounds}}]
Let $L^*$ be the minimum of the optimisation problem in Eq.~\eqref{eq:partial_identification_opt_fixed_order}, and let $L$ be the minimum of the LP task defined in Eq.~\eqref{eq:lp}. To prove the theorem, we first show that $L \leq L^*$, and then $L \geq L^*$.

Assume the optimal value $L^*$ is achieved by some SCM $\mathcal{M}^* \in \mathbb{M}_{\sigma}$ for query $P_{\mathcal{M}^*}(\gamma)$. By Thm.~\ref{thm:bounding_equivalence}, we can consider $\mathcal{M}^*$ to be an SCM parametrising the canonical PSCM $M_\sigma$. By construction of $\bm{\mathcal{S}}_\sigma$, the exogenous distribution of $\mathcal{M}^*$ maps to a valid joint probability distribution $\mathbf{q}^*$ over the states $\bm{s}_\sigma \in \bm{\mathcal{S}}_\sigma$. By Thm.~\ref{thm:ctf_linearity}, the objective query can be expressed as a linear combination of the probabilities associated with such states. Furthermore, since $\mathcal{M}^*$ is a feasible solution to Eq.~\eqref{eq:partial_identification_opt_fixed_order}, it must satisfy the observational constraint $P_{\mathcal{M}^*}(\bm{V}) = \hat{P}(\bm{V})$. Correspondingly, $\mathbf{q}^*$ satisfies the LP's constraint imposed by observational data as by Cor.~\ref{cor:obs_consistency}. Consequently, $\mathbf{q}^*$ constitutes a feasible solution within the constraints of the LP formulation in Eq.~\eqref{eq:lp}. Because the LP identifies the global minimum $L$ over all feasible distributions, it must be that $L \leq L^*$. 

Conversely, let $\mathbf{q}^*$ be the optimal solution that minimises the LP task in Eq.~\eqref{eq:lp}, such that the objective evaluates to $L$. Consider a specific SCM $\mathcal{M} \in \mathbb{M}_{\sigma}$; our argument will equivalently consider $\mathcal{M} \in \mathbb{M}_{M_{\sigma}}$ (after Thm.~\ref{thm:bounding_equivalence}). We define its exogenous distribution by assigning $P_{\mathcal{M}}(u) = q^*_{h(u)}$, where $h$ is the bijection established in Lem.~\ref{lemma:exogenous_signature_bijection}. Additionally, because the LP enforces the observational constraints, $\mathcal{M}$ matches the empirical distribution, then it is also a feasible candidate for the optimisation task in Eq.~\eqref{eq:partial_identification_opt_fixed_order}. Since $L^*$ represents the minimum over \emph{all} feasible SCMs in $\mathbb{M}_{\sigma}$, and we have found a specific feasible SCM $\mathcal{M}$ that achieves $L$, it must hold that $L^* \leq L$.

Since $L \leq L^*$ and $L^* \leq L$, we conclude that $L = L^*$. Furthermore, the explicit construction in the second half of the proof demonstrates that there exists an SCM $\mathcal{M} \in \mathbb{M}_{M_\sigma}$ that precisely attains the LP bounds.
\end{proof}

\begin{proof}[\textbf{Proof of Thm.~\ref{thm:reduced_lp_equivalence}}]
Let $Q_{LP}$ and $Q_{LLP}$ denote the feasible regions of the LP in Eq.~\eqref{eq:lp} and of the lifted LP in Eq.~\eqref{eq:lp_reduced}, respectively. 

For any feasible solution $\bm{q}^* \in Q_{LP}$, by definition of the equivalence relation in Eq.~\eqref{eq:signature_equivalence}, $q_{\bm{t}_\sigma} = \sum_{\bm{s}_\sigma \in \bm{t}_\sigma} q^*_{\bm{s}_\sigma}$ trivially satisfies the constraints of Eq.~\eqref{eq:lp_reduced} whilst preserving the objective value, as the Boolean evaluations are constant within a class $\bm{t}_\sigma \in \bm{\mathcal{T}}_\sigma$. 

Conversely, consider any feasible solution $\bm{q}^*_{\bm{t}_\sigma} \in Q_{LLP}$. Let $k_{\bm{t}_\sigma} \coloneq |\bm{t}_\sigma|$ be the number of signatures $\bm{s}_\sigma$ contained within a given equivalence class $\bm{t}_\sigma$. Note that by construction every class $\bm{t}_\sigma$ is non-empty, hence $k_{\bm{t}_\sigma} > 0$. We construct a valid distribution in the original space by assigning a uniform mass $q_{\bm{s}_\sigma} = q^*_{\bm{t}_\sigma} / k_{\bm{t}_\sigma}$ to every signature $\bm{s}_\sigma \in \bm{t}_\sigma$, for all $\bm{t}_\sigma \in \bm{\mathcal{T}}_\sigma$. Because the logical constraints evaluate identically across $\bm{t}_\sigma$, summing over this uniformly distributed mass recovers the constraints and the objective value of the lifted LP, satisfying Eq.~\eqref{eq:lp}.

Since these mappings strictly preserve both feasibility and the objective value in both directions, the extrema evaluated over $Q_{LLP}$ must exactly match those over $Q_{LP}$. Following the constructive argument detailed in the proof of Thm.~\ref{thm:exact_bounds}, there exists a witness SCM that attains these bounds, ensuring the tightness is inherited by the lifted LP in Eq.~\eqref{eq:lp_reduced}.
\end{proof}

\begin{proof}[\textbf{Proof of Thm.~\ref{thm:opt_equivalence_marginal}}]
Let $\mathcal{M}^* \in \mathbb{M}_\sigma$ and $\mathcal{\tilde{M}}^*\in \mathbb{M}_\sigma^{\bm{V}_\gamma}$ denote the SCMs giving the optima in, respectively, Eq.~\eqref{eq:partial_identification_opt_fixed_order} and Eq.~\eqref{eq:partial_identification_opt_fixed_order_marginal}.
Let also $L\coloneq P_{\mathcal{M}^*}(\gamma)$ and $L' \coloneq P_{\mathcal{\tilde{M}}^*}(\gamma)$. Following, for instance, \citet{bongers2021foundations}, we can marginalise out from $\mathcal{M}^*$ the variables in $\bm{V} \setminus \bm{V}_\gamma$ and obtain a well-defined SCM over $\bm{V}_\gamma$. Moreover, by construction, we have $P_\mathcal{M}^*(\bm{V})=\hat{P}(\bm{V})$, and this implies, in the marginal SCM, $P_\mathcal{M}^*(\bm{V}_\gamma)=\hat{P}(\bm{V}_\gamma)$. Overall, we have an SCM in 
$\mathbb{M}_\sigma^{\bm{V}_\gamma}$ giving $L$, and this proves $L \geq L'$. 

To prove the inverse inequality, let $P(\bm{V}'\mid\bm{V}_{\gamma}):= \hat{P}(\bm{V})/\hat{P}(\bm{V}_\gamma)$, where $\bm{V}':=\bm{V} \setminus \bm{V}_\gamma$. We augment the SCM $\mathcal{\tilde{M}}^*$ with the variables in $\bm{V}'$. By setting all the variables in $\bm{V}_\gamma$ as endogenous parents of those in $\bm{V}'$, we can easily obtain $P(\bm{V})=\hat{P}(\bm{V})$. Thus we have an SCM in $\mathbb{M}_\sigma$ such that $P(\bm{V})=\hat{P}(\bm{V})$ and giving $L'$. This proves $L \leq L'$, and hence the thesis.
\end{proof}

\begin{proof}[\textbf{Proof of Thm.~\ref{thm:order_invariance}}]
Let $\sigma \in \Sigma(\prec_\gamma)$ be an arbitrary total order compatible with $\prec_\gamma$, and consider the lifted LP in Eq.~\eqref{eq:lp_reduced} constructed over $\sigma$. Each equivalence class $\bm{t}_\sigma \in \bm{\mathcal{T}}_\sigma$ is uniquely characterised by an observational assignment $\bm{v} \in \bm{\mathcal{V}}$ and a Boolean query evaluation $b = \llbracket \bm{t}_\sigma \models \gamma \rrbracket \in \{0, 1\}$.

An equivalence class indexed by $\langle \bm{v}, b \rangle$ is non-empty (i.e., contains at least one signature $\bm{s} \in \bm{\mathcal{S}}_\sigma$) if and only if $\bm{v}$ and $b$ do not contradict the consistency axiom \citep{gallesAxiomaticCharacterizationCausal1998}. Specifically, for any atomic proposition $Y_x = y^*$ in $\gamma$, if $\bm{v}$ assigns $X = x$, consistency forces $Y_x = v_Y$. Hence, if $v_Y \neq y^*$, consistency prevents $\gamma$ from being satisfied, ruling out $b = 1$; conversely, if $\bm{v}$ forces every atomic term in $\gamma$ to hold, consistency rules out $b = 0$.

Crucially, whether a pair $\langle \bm{v}, b \rangle$ satisfies the consistency axiom depends solely on the assignment $\bm{v}$ and the syntax of $\gamma$, independently of the topological order $\sigma$. Moreover, because the canonical specification contains all possible functional responses, every pair $\langle \bm{v}, b \rangle$ that is compatible with consistency is realised by at least one signature $\bm{s} \in \bm{\mathcal{S}}_\sigma$: the observational response is fixed to $\bm{v}$, while the counterfactual responses under unobserved interventions can be set to satisfy $\gamma$ (yielding $b=1$) or violate at least one atom of $\gamma$ (yielding $b=0$).

We can conclude that the quotient space $\bm{\mathcal{T}}_\sigma$ of non-empty equivalence classes is invariant to $\sigma$. Indexing the lifted optimisation variables as $q_{\langle \bm{v}, b \rangle}$, the LP objective function is $\sum_{\bm{v}} q_{\langle \bm{v}, 1 \rangle}$ and the empirical constraints are $q_{\langle \bm{v}, 0 \rangle} + q_{\langle \bm{v}, 1 \rangle} = \hat{P}(\bm{v})$ for each $\bm{v} \in \bm{\mathcal{V}}$. Because the optimisation variables, constraints, and objective function are entirely decoupled from $\sigma$, the resulting LP is identical for all linear extensions $\sigma \in \Sigma(\prec_\gamma)$. Solving the LP for any arbitrary compatible $\sigma$ therefore yields the same bounds as Eq.~\eqref{eq:partial_identification_opt_fixed_order_marginal} and solves the global optimisation task in Eq.~\eqref{eq:partial_identification_over_orders}.
\end{proof}
\section{Additional Experiments}\label{sec:appendix_B}

\paragraph{ACE with Imperfect Compliance.}
\begin{table}[htp!]
\centering
\begin{tabular}{lrr}
\hline
Method & \multicolumn{1}{c}{Lower} & \multicolumn{1}{c}{Upper} \\
\hline
\multicolumn{3}{l}{\textit{\textbf{Vitamin A}}} \\
\hspace{3mm}\citet{balkeBoundsTreatmentEffects1997}& -0.195 & 0.005 \\
\hspace{3mm}This paper& -0.587 & 0.412 \\

\hline
\multicolumn{3}{l}{\textit{\textbf{Coronary}}} \\
\hspace{3mm}\citet{balkeBoundsTreatmentEffects1997}& 0.262  & 0.868 \\
\hspace{3mm}This paper& -0.685 & 0.971 \\
\hline
\end{tabular}
\caption{{Comparison of ACE bounds.} 
Lower and upper bounds for the ACE grouped by Study.}
\label{tab:ace}
\end{table}
We consider a model with three Boolean variables. In addition to the outcome $Y$, we distinguish between treatment exposure $X$ and treatment assignment $Z$. To evaluate the causal effect of the treatment when assigned, we focus on the query $\mathrm{ACE}_{X=0,X=1}(Y=1)$, defined as in Eq.~\eqref{eq:ace}. Observational information is available only in the form of a \emph{conditional} empirical distribution $\hat{P}(X,Y|Z)$. In our framework, consistency with the empirical distribution is expressed as:
\begin{equation}
\frac{\sum_{\bm{s}_{\sigma} \in \bm{\mathcal{S}}_\sigma} \llbracket \bm{s}_{\sigma} \models (x,y,z) \rrbracket \cdot q_{\bm{s}_{\sigma}}}{\sum_{\bm{s}_{\sigma} \in \bm{\mathcal{S}}_\sigma} \llbracket \bm{s}_{\sigma} \models z \rrbracket \cdot q_{\bm{s}_{\sigma}}}
= \hat{P}(x,y|z) \,,
\end{equation}
which yields linear constraints over the optimisation variables $q_{\bm{s}_\sigma}$, thus preserving the LP structure of our method. Although the query involves only $X$ and $Y$, the conditional nature of the empirical distribution prevents a marginalisation of $Z$. Consequently, the response signature must be based on all three Boolean variables. 
By Thm.~\ref{thm:order_invariance} we choose an arbitrary admissible order $\sigma$ such that $X \prec Y$.

We compare our bounds with those obtained from the LP formulation introduced by \citet{balkeBoundsTreatmentEffects1997} and later generalised by \citet{sachsGeneralMethodDeriving2023}. That approach requires an explicit causal graph. In particular, the model in Fig.~\ref{fig:balke} assumes that $Z$ affects $Y$ only through $X$, and that $Z$ is independent of an unobserved confounder $U$ affecting both $X$ and $Y$. Tab.~\ref{tab:ace} reports the resulting comparison for two observational studies: the \emph{vitamin A} supplementation study in northern Sumatra and the \emph{coronary} primary prevention trial. Our bounds remain informative, but are substantially wider than those of \citet{balkeBoundsTreatmentEffects1997}. This difference reflects the fact that our framework relies only on the order restrictions induced by the query, whereas the graphical model in Fig.~\ref{fig:balke} imposes additional structure, most notably the exclusion of a direct effect from $Z$ to $Y$. To model these effects with such a granularity one would either impose additional (non-linear) constraints or eventually require a causal graph, but this lies outside manuscript's scope. By Thm.~\ref{thm:exact_bounds}, the more extreme values admitted by our procedure are attainable under SCMs compatible with our weaker assumptions, including models in which $Z$ may directly affect $Y$.

As pointed out by \citet{balkeCounterfactualProbabilitiesComputational1994}, bounds can be further tightened if additional assumptions are made about subject behaviours: imposing a monotonic constraint, e.g. a positive response on treatment, $Y_{x_1,z} \geq Y_{x_0,z}$ for every $z \in \mathcal{Z}$. As a consequence, negative lower bounds on the effects are raised to zero. Nevertheless, even when combined with additional assumptions such as weak exogeneity, these restrictions are still insufficient to recover the analytical bounds implied by the graphical model, as discussed in Sect.~\ref{sec:examples}.

\paragraph{PNS Generalisation for Non-Binary Variables}
\begin{figure}[htbp]
    \centering
    \begin{tikzpicture}
        \begin{axis}[
            title={},
            xlabel={Cardinality $n$},
            ylabel={Computation Time (seconds)},
            ymode=log,
            log basis y=10,
            xmin=3.5, xmax=15.5,
            grid=major,
            legend pos=north west,
            width=0.8\textwidth,
            height=8cm,
            cycle list name=color list
        ]
        
        \addplot+[mark=*, color=gray] coordinates {
            (4, 0.0579) (5, 0.0286) (6, 0.0366) (7, 0.0378)
            (8, 0.0523) (9, 0.0771) (10, 0.0950) (11, 0.1416)
            (12, 0.2084) (13, 0.3179) (14, 0.4744) (15, 0.6226)
        };
        \addlegendentry{GPNS(2)}
        
        \addplot+[mark=square*, color=gray] coordinates {
            (4, 0.0275) (5, 0.0312) (6, 0.0515) (7, 0.0831)
            (8, 0.1737) (9, 0.3366) (10, 0.6532) (11, 1.2209)
            (12, 2.0519) (13, 3.6114) (14, 5.8122) (15, 8.4889)
        };
        \addlegendentry{GPNS(3)}
        
        \addplot+[mark=triangle*, color=gray] coordinates {
            (4, 0.0292) (5, 0.0418) (6, 0.1450) (7, 0.3541)
            (8, 1.0529) (9, 2.5386) (10, 5.6041) (11, 11.8192)
            (12, 24.0566) (13, 45.8219) (14, 77.9922) (15, 128.9642)
        };
        \addlegendentry{GPNS(4)}
        \end{axis}
    \end{tikzpicture}
    \caption{End-to-end computation time for $\mathrm{GPNS}(r)$ bounds, on a semi-logarithmic scale, as a function of the variable cardinalities
    ($n = |\mathcal{X}| = |\mathcal{Y}|$).}
    \label{fig:pns_timings}
\end{figure}
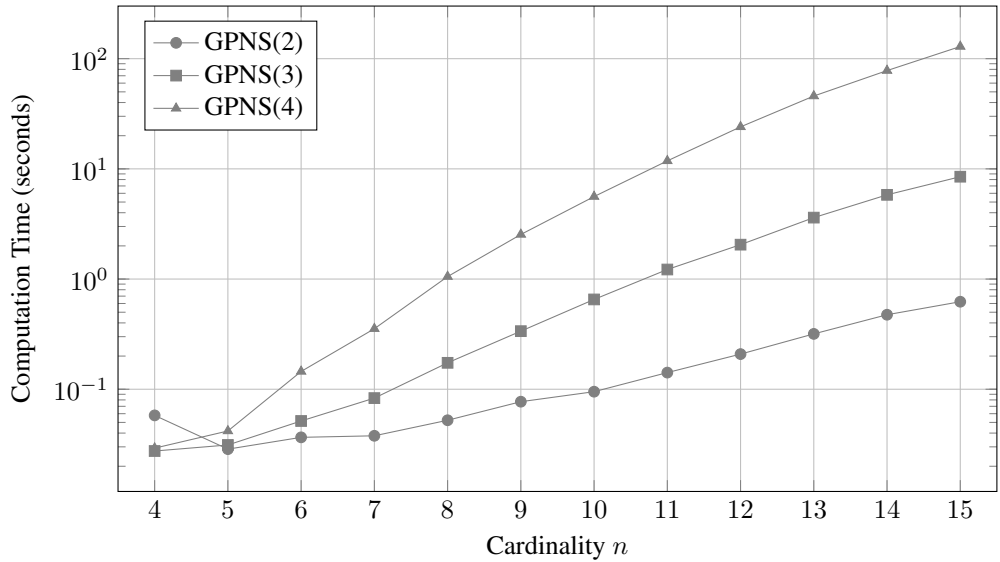

To assess the practical scalability of the lifted LP in Eq.~\eqref{eq:lp_reduced}, we measured the end-to-end time required to
construct the signature equivalence classes, aggregate the lifted
variables, and solve the resulting LP. The benchmark query is a
multi-hypothetical generalisation of PNS for ordinal treatment and
outcome variables, in the spirit of \citet{kawakami24aPoCContinuosVector}. Let $X$ and $Y$ be ordinal treatment and effect variables, respectively, defined over finite domains $\mathcal{X} = \{x^{(0)}, \dots, x^{(n-1)}\}$ and $\mathcal{Y} = \{y^{(0)}, \dots, y^{(n-1)}\}$ where $n$ is an experimental parameter. The GPNS of order $r$ (where $2 \leq r \leq n$) evaluates the joint probability of strict compliance across $r$ distinct interventions, defined as:
\begin{equation}
\label{eq:gpns}
    \mathrm{GPNS}(r) \coloneq P\left(\bigwedge_{i=0}^{r-1} Y_{x^{(i)}} = y^{(i)}\right) \,.
\end{equation}
For binary $X$ and $Y$, $\mathrm{GPNS}(2)$ reduces to the
standard PNS event. For each $n\in\{4,\ldots,15\}$, we randomly parametrise a Bayesian network with the endogenous structure in Fig.~\ref{fig:2}, estimated $\hat P(X,Y)$ from simulated observational data, and computed bounds for $\mathrm{GPNS}(r)$ with $r\in\{2,3,4\}$. Figure~\ref{fig:pns_timings} reports the resulting end-to-end computation times.

For this benchmark, $\bm V_\gamma=\{X,Y\}$, so the lifted LP has at most
$|\mathcal X||\mathcal Y|=n^2$ variables, as derived in Sect.~\ref{sec:algorithm}. Results are consistent with the polynomial scaling of the lifted formulation, although we must consider the overhead introduced by the symmetry detection, LP construction, and numerical optimisation. As expected, solving the LP remains tractable for standard solvers as the domain expands.

\section{Canonical Specifications}\label{sec:appendix_C}
\begin{figure}[!t]
  \centering
  \begin{subfigure}[b]{0.3\textwidth}
    \centering
    \begin{tikzpicture}[scale=1.0]
        \node[dot,label=below:$X$] (x) at (0,0) {};
        \node[dot,label=below:$W$] (w) at (1.25,0) {};
        \node[dot,label=below:$Y$] (y) at (2.5,0) {};
        \node[dot2,label=above:$U_X$] (ux) at (0,1) {};
        \node[dot2,label=above:$U_{W}$] (uw) at (1.25,1) {};
        \node[dot2,label=above:$U_{Y}$] (uy) at (2.5,1) {};
        \draw[a2] (x) -- (w);
        \draw[a2] (w) -- (y);
        \draw[a2] (x) .. controls (1.25,0.55) .. (y);
        \draw[a2] (ux) -- (x);
        \draw[a2] (uw) -- (w);
        \draw[a2] (uy) -- (y);
    \end{tikzpicture}
    \caption{}
    \label{fig:markovian-dag}
  \end{subfigure}%
  \hfill
  \begin{subfigure}[b]{0.3\textwidth}
    \centering
    \begin{tikzpicture}[scale=1.0]
        \node[dot,label=below:$X$] (x) at (0,0) {};
        \node[dot,label=below:$W$] (w) at (1.25,0) {};
        \node[dot,label=below:$Y$] (y) at (2.5,0) {};
        \node[dot2,label=above:$U_X$] (ux) at (0,1) {};
        \node[dot2,label=above:$U_{WY}$] (uwy) at (1.875,1) {};
        \draw[a2] (x) -- (w);
        \draw[a2] (w) -- (y);
        \draw[a2] (x) .. controls (1.25,0.55) .. (y);
        \draw[a2] (ux) -- (x);
        \draw[a2] (uwy) -- (w);
        \draw[a2] (uwy) -- (y);
    \end{tikzpicture}
    \caption{}
    \label{fig:semimarkovian-dag-partial}
  \end{subfigure}%
  \hfill
  \begin{subfigure}[b]{0.3\textwidth}
    \centering
    \begin{tikzpicture}[scale=1.0]
        \node[dot,label=below:$X$] (x) at (0,0) {};
        \node[dot,label=below:$W$] (w) at (1.25,0) {};
        \node[dot,label=below:$Y$] (y) at (2.5,0) {};
        \node[dot2,label=above:$U$] (u) at (1.25,1) {};
        \draw[a2] (x) -- (w);
        \draw[a2] (w) -- (y);
        \draw[a2] (x) .. controls (1.25,0.55) .. (y);
        \draw[a2] (u) -- (x);
        \draw[a2] (u) -- (w);
        \draw[a2] (u) -- (y);
    \end{tikzpicture}
    \caption{}
    \label{fig:semimarkovian-dag-complete}
  \end{subfigure}
  \caption{{Examples of canonical PSCM specifications compatible with the endogenous order $\sigma \coloneq (X,W,Y)$.} (a)~Markovian specification. (b)~Specification with one exogenous variable shared by $W$ and $Y$. (c)~Specification with one exogenous variable shared by all endogenous variables. Grey nodes denote exogenous variables.}
  \label{fig:example-dags}
\end{figure}
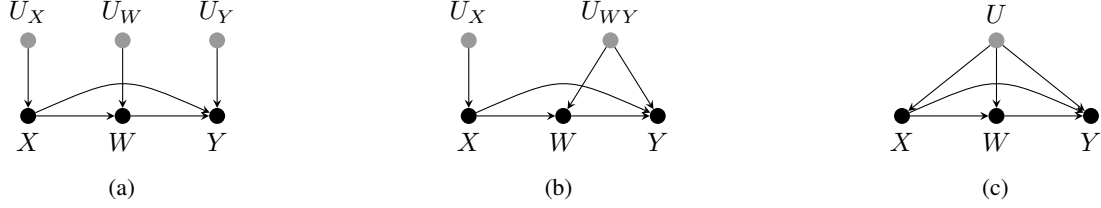
We expand the discussion of \emph{canonicalisation} of SCMs and PSCMs briefly introduced in Sect.~\ref{sec:algorithm}. We recall here that the canonical PSCM used in the proof of Thm.~\ref{thm:bounding_equivalence} builds upon the theoretical contributions of \citet{zhangPartialCounterfactualIdentification2021} and \citet{balkeCounterfactualProbabilitiesComputational1994}. Generally speaking, an SCM takes its structural equations (SEs) as given. When these are unknown and the endogenous variables take on finitely many values, the literature shows that one may still specify a PSCM from the causal graph alone by adopting a \emph{canonical} representation.

For an endogenous variable $V \in \bm{V}$ with parents $\mathrm{Pa}_V$, a state of its exogenous parent $U_V$ acts as a selector for a deterministic mapping (or \emph{response function}) from the domain of $\mathrm{Pa}_V$ to $V$. There are exactly $|\mathcal{V}|^{|\mathrm{Pa}_V|}$ such possible functions, where $|\mathrm{Pa}_V| \coloneq \prod_{Z \in \mathrm{Pa}_V} |\mathcal{Z}|$. In the absence of unobserved confounding, as in Markovian models, each $U_V$ is independent of all other exogenous variables and can be replaced without loss of generality by a discrete canonical variable of finite cardinality $|\mathcal{U}_V| = |\mathcal{V}|^{|\mathrm{Pa}_V|}$. When latent confounding is present, as in semi-Markovian models, unobserved variables may jointly affect multiple endogenous variables. The graph is then partitioned into \emph{confounded components} (or \emph{c-components}) \citep{tian2002studies}, defined as the maximal sets of endogenous variables connected by paths of latent confounders (see \emph{confounded paths} in \citeauthor{shpitserWhatCounterfactualsCan2007}, \citeyear{shpitserWhatCounterfactualsCan2007}). In a canonical specification, all latent confounding within each c-component $\bm{C} \subseteq \bm{V}$ is replaced without loss of generality by a single discrete exogenous variable $U_{\bm{C}}$ that jointly selects the response functions for every variable in $\bm{C}$, whilst the exogenous variables associated with distinct c-components remain mutually independent. Because the state space of $U_{\bm{C}}$ must encompass all joint assignments of response functions, its canonical space $\mathcal{U}_{\bm{C}}$ indexes all possible combinations of individual mappings, yielding $|\mathcal{U}_{\bm{C}}| = \prod_{V \in \bm{C}} |\mathcal{V}|^{|\mathrm{Pa}_V|}$. Canonicalisation guarantees that any discrete SCM (or PSCM) admits a discrete equivalent inducing an identical set of observational and counterfactual distributions \citep{evansGraphsMarginsBayesian2016,duarte2024automated}.

To ground the ideas, we proceed by means of an example. Consider $\bm{V} \coloneq \{X, W, Y\}$ with binary domains $\mathcal{X} = \mathcal{Y} = \{0, 1\}$ and ternary domain $\mathcal{W} = \{0, 1, 2\}$. We assume the causal order $X \prec W \prec Y$. The three panels in Fig.~\ref{fig:example-dags} share the same endogenous structure compatible with this order. A canonical specification of these causal diagrams is obtained by enumerating the deterministic response functions for each endogenous variable. 
First, consider the unconfounded setting in Fig.~\ref{fig:markovian-dag}, where every variable forms a singleton c-component. Since $\mathrm{Pa}_X = \emptyset$, there are two constant mappings, giving $|\mathcal{U}_X| = 2$. For $W$, whose parent is $X$, the mechanism $w \gets f_W(x, u_W)$ assigns a state in $\mathcal{W}$ to each $x \in \mathcal{X}$; hence, $|\mathcal{U}_W| = |\mathcal{W}|^{|\mathcal{X}|} = 3^2 = 9$. Finally, the mapping $y \gets f_Y(x, w, u_Y)$ assigns a value in $\mathcal{Y}$ to every configuration $(x, w) \in \mathcal{X} \times \mathcal{W}$, yielding $|\mathcal{U}_Y| = |\mathcal{Y}|^{|\mathcal{X}| \cdot |\mathcal{W}|} = 2^{2 \cdot 3} = 64$. Next, consider Fig.~\ref{fig:semimarkovian-dag-partial}, where unobserved confounding between $W$ and $Y$ induces the partition into c-components $\{X\}$ and $\{W, Y\}$. The canonical confounder $U_{WY}$ jointly indexes the mappings of both $W$ and $Y$, yielding $|\mathcal{U}_{WY}| = |\mathcal{U}_W| \cdot |\mathcal{U}_Y| = 9 \cdot 64 = 576$, whilst $|\mathcal{U}_X| = 2$. Lastly, in the fully confounded model of Fig.~\ref{fig:semimarkovian-dag-complete}, all endogenous variables belong to a single c-component $\bm{C} = \bm{V}$. A single global exogenous parent $U$ governs all mechanisms, with cardinality $|\mathcal{U}| = |\mathcal{U}_X| \cdot |\mathcal{U}_W| \cdot |\mathcal{U}_Y| = 2 \cdot 9 \cdot 64 = 1{,}152$. The canonical PSCM constructed in the proof of Thm.~\ref{thm:bounding_equivalence} matches this full-confounding specification.
\end{document}